\documentclass{article}
\usepackage{fullpage}
\usepackage{smile}
\RequirePackage{natbib}

\usepackage[pdftex,colorlinks,linkcolor=blue,citecolor=blue,filecolor=blue,urlcolor=blue]{hyperref}      
\usepackage{url}           
\usepackage{booktabs}      
\usepackage{amsfonts}       
\usepackage{nicefrac}       
\usepackage{xcolor}

\usepackage{subfigure}
\usepackage{multirow}
\usepackage{color}
\usepackage{amsmath, amsthm, amssymb, amsfonts}

\usepackage{xspace}
\usepackage{comment}
\usepackage{dsfont}
\usepackage{subfigure}
\usepackage{algorithm}
\usepackage{algorithmic}
\usepackage{multirow}
\RequirePackage{times}
\usepackage{comment}
\usepackage{xcolor}

\usepackage{authblk}
\author[1]{Jialin Dong}
\author[2]{Lin F. Yang}

\affil[1,2]{University of California, Los Angeles}
{
	\makeatletter
	\renewcommand\AB@affilsepx{, \protect\Affilfont}
	\makeatother
	\affil[1]{jialind@g.ucla.edu}
	\affil[2]{linyang@ee.ucla.edu}
	
}

\title{ Does Sparsity Help in Learning Misspecified Linear Bandits?}

\begin{document}

	\date{}
	\maketitle

	\begin{abstract}
		Recently, the study of linear misspecified bandits has generated intriguing implications of the hardness of learning in bandits and reinforcement learning (RL). In particular, \citet{du2020good} show that even if a learner is given linear features in $\mathbb{R}^d$ that approximate the rewards in a bandit or RL with a uniform error of $\varepsilon$, searching for an $O(\varepsilon)$-optimal action requires pulling at least $\Omega(\exp(d))$ queries. Furthermore, \citet{lattimore2020learning} show that a degraded $O(\varepsilon\sqrt{d})$-optimal solution can be learned within $\poly(d/\varepsilon)$ queries. Yet it is unknown whether a structural assumption on the ground-truth parameter, such as sparsity, could break the $\varepsilon\sqrt{d}$ barrier. In this paper, we address this question by showing that algorithms can obtain $O(\varepsilon)$-optimal actions by querying $O(\varepsilon^{-s}d^s)$ actions, where $s$ is the sparsity parameter, removing the $\exp(d)$-dependence.
	We then establish information-theoretical lower bounds, i.e., $\Omega(\exp(s))$, to show that our upper bound on sample complexity is nearly tight if one demands an error $ O(s^{\delta}\varepsilon)$ for $0<\delta<1$. For $\delta\geq 1$, we further show that $\poly(s/\varepsilon)$ queries are possible when the linear features are ``good'' and even in general settings.
		These results provide a nearly complete picture of how sparsity can help in misspecified bandit learning and provide a deeper understanding of when linear features are ``useful'' for bandit and reinforcement learning with misspecification. 
		
	\end{abstract}
	
	\section{Introduction}\label{intro}
	Bandit and reinforcement learning problems in real-world applications, e.g., autonomous driving \citep{kiran2021deep}, healthcare \citep{esteva2019guide}, recommendation system \citep{bouneffouf2012contextual}, marketing and advertising \citep{schwartz2017customer}, are challenging due to the magnificent state/action space. To address this challenge, a function approximation framework has been introduced, which first extracts feature vectors for state/action space and then approximates the value functions of all policies in RL (or the reward functions of all actions in bandit problems) with feature representations. In some real-world applications, feature representations may not have vanilla linear mapping. In these scenarios, a linear feature representation can approximate the value functions (or the reward functions) with a small uniform error known as misspecification. Unfortunately, \citet{du2020good} show that searching for an $O(\varepsilon)$-optimal action in these scenarios requires pulling at least $\Omega(\exp(d))$ queries. However, if we relax the goal of finding $O(\varepsilon)$-optimal action, there is still a chance. Instead, \citet{lattimore2020learning} find an action that is suboptimal with an error of at most $O(\varepsilon\sqrt{d})$ within $\poly(d/\varepsilon)$ queries, where $d$ is the dimension of the feature vectors.
	
	By scrutinizing the novel result proposed by \citet{lattimore2020learning}, the dependence on $\sqrt{d}$ raises concern regarding the potential blowup of the approximation error. We are modestly optimistic that some structural patterns, such as sparsity, in feature representation schemes are beneficial to break the $\varepsilon\sqrt{d}$ barrier. This idea comes from a vast literature that studies high-dimensional statistics in sparse linear regression \citep{buhlmann2011statistics,wainwright2019high} and successfully applies it to sparse linear bandits \citep{sivakumar2020structured,abbasi2012online,bastani2020online,wang2018minimax,su2020doubly,lattimore2015linear}. Moreover, the sparsity-structure in linear bandits are meaningful and crucial to many practical applications where there are many potential features but no apparent evidence on which are relevant, such as personalized health care and online advertising \citep{carpentier2012bandit,abbasi2012online}. 
	The essential difference in sparse linear bandits between our paper and state-of-the-art is the study of the possible model misspecification, i.e., the ground truth reward means might be an $\varepsilon$ error away from a sparse linear representation for any action.

	Model misspecification is widely seen in practice and has been widely studied only in the dense model (also known as misspecified linear bandits) \citep{bogunovic2021misspecified,takemura2021parameter,zanette2020learning,wang2020reinforcement}, where the best polynomial-sample algorithm suffers a $O(\varepsilon\sqrt{d})$ estimation error, which can be prominent when the feature dimension $d$ is sufficiently large.
	However, it is unexplored whether a structural sparsity assumption on the ground-truth parameter could break the $\varepsilon\sqrt{d}$ barrier. Additionally, there is little understanding of the conditions when linear features are ``useful'' for bandit problems and reinforcement learning with misspecification.
	\paragraph{Contribution.} 
	\begin{itemize}
		\item We establish novel algorithms that obtain $O(\varepsilon)$-optimal actions by querying ${O}(\varepsilon^{-s}d^s)$ actions, where $s$ is the sparsity parameter. 
		For fixed sparsity $s$, the algorithm finds an $O(\varepsilon)$-optimal action with $\poly(d/\varepsilon)$ queries, breaking the $O(\varepsilon\sqrt{d})$ barrier.
		The $\varepsilon^{-s}$ dependence in the sample bound can be further improved to $\tilde{O}(s)$ if we allow an $O(\varepsilon\sqrt{s})$ suboptimality.
		
		\item We establish information-theoretical lower bounds to show that our upper bounds are nearly tight. In particular, we show that any sound algorithms that can obtain $O(\Delta)$-optimal actions  need to query ${\Omega}(\exp(m\varepsilon/\Delta))$ samples from the bandit environment, where the approximate error $\Delta$, defined in Definition \ref{delta}, satisfies $\Delta\geq\varepsilon$. Hence, for approximation error of the form $O(s^{\delta}\varepsilon)$, for any $0<\delta<1$, $\exp(s)$-dependence in the sample complexity is not avoidable.
		
		\item We further break the $\exp(s)$ sample barrier by showing an algorithm that achieves $O(s\varepsilon)$ sub-optimal actions while only querying $\poly(s/\varepsilon)$ samples in the regime the action features possess specific benign structures (hence ``good'' features).  We then relax the benign feature requirement to arbitrary feature settings and propose an algorithm with efficient sample complexity of $\poly(s/\varepsilon)$.
	\end{itemize}

	In summary, our results provide a nearly complete picture of how sparsity can help in misspecified bandit learning and provide a deeper understanding of when linear features are ``useful'' for the bandit and RL with misspecification.

	\section{Related work}
	This section summarizes the state-of-the-art in several areas of interest related to our work: function approximation, misspecified feature representation, and sparsity in bandits and reinforcement learning.
	
	\paragraph{Function approximation in bandits and reinforcement learning} Function approximation schemes that approximate value functions in RL (reward function in bandit problem) with feature representations are widely used for generalization across large state/action spaces. A recent line of work studies bandits \citep{ding2021efficient,russo2013eluder,dani2008stochastic,chu2011contextual} and RL with linear function
	approximation \citep{jin2020provably,zanette2020learning,cai2020provably,zanette2020provably,agarwal2020flambe,neu2020unifying}.  Beyond the linear setting, there is a flurry line of
	research studying RL with general function approximation \citep{wang2020optimism,osband2014model,jiang2017contextual} and bandits with general function approximation \citep{li2017provably,kveton2020randomized,filippi2010parametric,jun2017scalable,foster2021instance}. The regret upper  bound $O(\poly(d)\sqrt{n})$ can be achieved in the above papers, where $d$ is the ambient dimension (or complexity measure such as eluder dimension) of the feature space and $n$ is the number of rounds.
	
	\paragraph{Misspecified bandits and reinforcement learning}	Recently, interest has been aroused to deal with the situation when the value function in RL (or the rewards functions in bandits) is approximated by a linear function where the approximation
	error is at most $\varepsilon$, also known as the misspecified linear bandit and reinforcement learning.
	The misspecification facilitates us to establish a more complicated reward function than a linear function. For instance, it enables the characterization of a reward function that may change over the rounds,  which is common in real-world applications such as education, healthcare, and recommendation systems \citep{chu2011contextual}.
	
	\citet{du2020good} showed that no matter whether value-based learning or model-based learning, the agent needs to sample an exponential number of trajectories to find an $O(\varepsilon)$-optimal policy for reinforcement learning with $\varepsilon$-misspecified linear features. This result shows that good features (e.g., linear features with small misspecification) are not sufficient for sample-efficient RL if the approximation error guarantee is close to the misspecification error.
	By relaxing the objective of achieving $O(\varepsilon)$-optimality, \citet{lattimore2020learning} showed that $\poly(d/\varepsilon)$ samples are sufficient to obtain an $O(\varepsilon\sqrt{d})$-optimal policy (in the simulator model setting of RL), where $d$ is the feature dimension, indicating the same features are ``good'' in a different requirement. The hard instances used in both papers are in fact bandit instances and hence provide understanding for misspecified linear bandit problems as well.

	A number of works  in the literature, such as \citep{foster2020adapting, vial2022improved,takemura2021parameter,wei2022model, jin2020provably}, can also deal with misspecification in linear bandits or RL with linear features. 
	These algorithms can only achieve a $O(\varepsilon\sqrt{d})$ error guarantee at best (when their regret bounds are translated to PAC bounds) with $\poly(d/\varepsilon)$ samples.

	\paragraph{Sparse linear bandits and reinforcement learning} In this section, we briefly review the literature on the sparse linear bandits and RL, where no misspecification is considered. We also note that these results are stated in regret bounds, which can be easily converted to PAC bounds.

	\citet{abbasi2012online} proposed an 
	online-to-confidence-set conversion approach which achieves a regret upper bound  of  $O(\sqrt{sdn})$, where $s$ is a known parameter on the sparsity. A matching lower bound is given in \citep{lattimore2020bandit}[Chapter 24.3], which shows that polynomial dependence on $d$ is generally unavoidable without additional assumptions. 
	To address this limitation, another line of literature \citep{kim2019doubly,bastani2020online,wang2018minimax} studied the sparse contextual linear bandits where the action set is different in each round and follows some context distribution.
	\citet{kim2019doubly} developed a doubly-robust Lasso bandit approach with an $O(s\sqrt{n})$ upper bound. \citet{bastani2020online} considered 
	the scenario where each arm has an underlying parameter and derived a $O( K s^2(\log(n))^2)$ upper bound
	which was improved to $O(K s^2\log(n))$ by \citet{wang2018minimax}, where $K$ is the number of arms. 
	\citet{sivakumar2020structured} proposed a structured greedy algorithm to achieve an $O(s\sqrt{n})$ upper bound. 
	\citet{hao2020high} derived a $\Omega(n^{2/3})$ minimax regret lower bound for sparse linear bandits where the feature vectors lack a well-conditioned exploration distribution.

	There are many previous works studying feature selection in reinforcement learning. Specifically, \citet{kolter2009regularization, geist2011, painter2012greedy, liu2012regularized} proposed algorithms with $\ell_1$-regularization for temporal-difference (TD) learning.  \citet{ghavamzadeh2011finite} and \citet{geist2012dantzig} proposed Lasso-TD to estimate the value function in sparse reinforcement learning and derived finite-sample MDP statistical analysis. \citet{hao2021sparse} provided nearly optimal statistical analysis of high dimensional batch reinforcement learning (RL)
	using sparse linear function approximation. \citet{ibrahimi2012efficient}
	derived an $O(d\sqrt{n})$ regret bound in high-dimensional sparse linear quadratic systems where $d$ is the dimension of the state space. The hardness of online reinforcement learning in fixed horizon has been studied by \citet{hao2021online}, which shows that linear regret
	is generally unavoidable in this case, even if there exists a policy that collects well-conditioned data.

	\section{Preliminary}
	Throughout this paper, $f({n}) = O(g(n))$  denotes that there exists a constant $c>0$ such that $|f(n)|\leq c|g(n)|$ and $\tilde{O}(\cdot)$ ignores poly-logarithmic factors. $f(n)=\Omega( g(n))$ means that there exists a constant $c>0$ such that $|f(n)|\geq c|g(n)|$. In addition, the notation $f(n)=\Theta( g(n))$ means that there exists constants $c_1, c_2>0$ such that $c_1|g(n)|\leq |f(n)|\leq c_2|g(n)|$. For a given integer $n$, let $[n]$ denote the set $\{1,\cdots, n\}$. Let $C> 0$ denote a suitably universal large constant. For a matrix $A\in\R^{m\times n}$, the set of rows is denoted by $\rows(A)$.
	Define an index set $\cM\subseteq[d]$ such that $|\cM| =s$. Let $\Phi_{\cM}\in\R^{k\times s}$ be the submatrix of $\Phi\in\R^{k\times d}$ and $\theta_{\cM}\in \R^s$ be the sub-vector of $\theta\in\R^d$.
	
	Consider a  bandit problem where the expected rewards are nearly a linear function of their associated features. Let $\Phi \in \R^{k \times d}$ denote the feature matrix whose rows are feature vectors corresponding to $k$ actions. 
	In rounds $t \in [n]$, the agent chooses actions $(a_t)_{t=1}^n$ with $a_t \in \rows(\Phi)$ and receives a reward
	\begin{equation}\label{def:sparse_linear}
	r_{a_t} = \langle a_t, \theta^*\rangle + \nu_{a_t} \,,
	\end{equation}
	where $\nu_{a_t} \in[-\varepsilon,\varepsilon]$, $\varepsilon>0$ for $t\in [n]$ and $\theta^*\in\mathbb R^d$ is an unknown parameter vector. 
	We only consider \emph{deterministic} rewards as small unbiased noises from rewards that do not change the sample complexity analysis of this paper by much but complicate the presentation. In Appendix \ref{noise}, we provide additional discussion on the noisy setting of the rewards.
	
	We make the mild boundedness assumption for each element of the feature matrix such that $\rows(\Phi)\in \mathbb{S}^{d-1}_B$.
	The parameter vector $\theta^*$ is assumed to be $s$-sparsity: 
	\begin{equation*}
	\|\theta^*\|_0 = \sum_{j=1}^d \mathds{1}\{\theta^*_j\neq 0\} =s ~~\text{and}~~\norm{\theta^*}_2 \leq 1.
	\end{equation*}
	We also assume that $\forall~x\in \rows(\Phi)$, there is $\norm{x}_2\leq 1$.
	\section{Main Results}
	In this section, we first present an $O(\varepsilon)$-optimal algorithm that takes $O(\varepsilon^{-s}d^s)$ queries in Section \ref{m-sparse} for $\varepsilon$-misspecified  $s$-sparse linear bandit. 
	Then we derive a nearly matching lower bound in Section~\ref{lower-bound}. 
	\subsection{An Algorithm that Breaks the $\Omega(\exp(d))$ Sample Barrier}\label{m-sparse}
	The core idea of our algorithm is based on an elimination-type argument. In particular, we would guess an estimator $\hat{\theta}$  for $\theta^*$ and a index set $\cM\subset [d]$. Then for each guess of $\hat{\theta}$ and $\cM$, we check the actions that have similar features restricting to $\cM$. Querying an action in this group allows us to rule out the guess of $\cM$ and $\hat{\theta}$ if they were not correct. If the ground truth $\theta^*$ is dense, this algorithm would take $\Omega(\exp(d))$ queries. Fortunately, since $|\cM|=s$, we can establish an $O(\varepsilon)$-net with a small size and eliminate the incorrect parameters in an efficient fashion. Below, we present the algorithm more formally.
	
	Define an index set $\cM\subseteq[d]$ such that $|\cM| =s$. Let $\cM^*$ denote the non-zero subset of $\theta^*$.  
	Denote $\cN^s$ as a maximal $\varepsilon/2$-separated subset of the Euclidean sphere $\mathbb{S}^{s-1}$ with radius of $1$ . The set $\cN^s$ satisfies that 
	$\norm{x-y}_2 \geq \varepsilon/2, $
	for all $x,y \in\cN^s$, and no subset of  $\mathbb{S}^{s-1}$ containing $\cN^s$ satisfies this condition. Thus, the size of $\cN^s$ is 
	\begin{align}\label{size}
	|\cN^s| \leq \left(\frac{4}{\varepsilon}+1\right)^s.
	\end{align}
	For a set $\cM$, we denote an estimator as $\hat{\theta}_{\cM}\in \cN^s$ to indicate the estimator which has only non-zero coordinates at $\cM$.
	
	For $\forall w\in\mathcal{N}^s$, we collect all $x\in\rows(\Phi)$ close to $w$ by the measurement $|\hat{\theta}_{\cM}^{\top}(x_{\cM}-w)|$ where $x_{\cM}\in\RR^s$ is the sub-vector of $x\in\RR^d$ restricted to the index set $\cM$ and define the set as
	\begin{align}\label{set}
	\cR_{\cM}^w(\hat{\theta}_{\cM}): = \{x\in\rows(\Phi):|\hat{\theta}_{\cM}^{\top}(x_{\cM}-w)|\leq \frac{\varepsilon}{2}\}.
	\end{align}
	The above set is simply denoted as $\cR_{\cM}^w$ in the following proof if $\hat{\theta}_{\cM}$ is clear from the context.
	In each round of the algorithm, we find $x\in \cR_{\cM}^w$ and a set $\cM'$ ($\cM'\neq \cM$) such that $\hat{\theta}_{\cM'}^{\top}x_{\cM'}$ deviates from $\hat{\theta}_{\cM}^{\top}w$ (at least $\Omega(\varepsilon)$). Then, we query such $x$ and receive the corresponding reward $r_x$. By comparing the difference between $r_x$ and $\hat{\theta}_{\cM}^{\top}w$, we can know whether the subset $\cM$ or $\cM'$ of $x$ is more likely to determine the reward $r_x$ and rule out the incorrect parameters. For $x\in \cR_{\cM}^w$, let $[x]_{\cN^s}$ denote the  vector $v=\arg\min_{w\in\cN^s}\norm{w-x_{\cM}}_2$ where $x_{\cM}\in\RR^s$ is the sub-vector of $x$.  Let $(\sim,\cM,\hat{\theta}_{\cM})\in\cS$ denote all of the elements involving the index set $\cM$ and $\hat{\theta}_{\cM}\in \cN^s$. We present the full algorithm in Algorithm \ref{alg2}.

	\begin{algorithm}[htb!]		\caption{Parameter Elimination}
		\begin{algorithmic}[1]\label{alg2}
			\STATE
			\textbf{Input:} feature matrix $\Phi\in\RR^{k\times d}$.
			\STATE
			\textbf{Initialize:} $\mathcal{S}:=  \{(w,\cM,\hat{\theta}_{\cM}): w \in \cN^s, \cM\subseteq[d], |\cM|=s, \hat{\theta}_{\cM}\in \cN^s\}$.
			\STATE For each {$(w,\cM,\hat{\theta}_{\cM})\in\cS$}, establish $\cR_{\cM}^w$ as (\ref{set}). 
			
			\WHILE{ there exit $(w,\cM,\hat{\theta}_{\cM})\in\cS$,  $\cM'\subseteq[d], |\cM'|=s$, $\cM\neq\cM'$, and $x\in \cR_{\cM}^w$ such that $(\sim,\cM',\hat{\theta}_{\cM'})\in\cS$, $|\langle x_{\cM'},\hat{\theta}_{\cM'}\rangle-\langle w,\hat{\theta}_{\cM}\rangle|>5\varepsilon/2$}\label{con}
			\STATE Query the action $x$ and receive a reward $r_x = \langle x, \theta^*\rangle + \nu_x$ where $\nu_x\in[-\varepsilon,\varepsilon]$.
			\STATE If $|r_x-\langle w,\hat{\theta}_{\cM}\rangle|> 3\varepsilon/2$ then $\cS = \cS \backslash{(\sim,\cM,\hat{\theta}_{\cM})}$, otherwise $\cS = \cS \backslash{(\sim,\cM',\hat{\theta}_{\cM'})}$.
			\ENDWHILE\label{end_con}

			\STATE  Find a certain set $\cL\subseteq[d], |\cL|=s$ and corresponding $\hat{\theta}_{\cL}\in\cN^s$ such that $(\sim,\cL,\hat{\theta}_{\cL})\in\cS$.
			
			\STATE \textbf{Output:} $\hat \theta_{\cL}$ and $\cL$
			
		\end{algorithmic}
	\end{algorithm}
	
	\begin{theorem}\label{theo2}
		After \[O\left(\left(\frac{1}{\varepsilon}\right)^{s}\cdot{d \choose s}\right)\] number of queries, the outputs of Algorithm \ref{alg2}, $\hat \theta_{\cL}$ and $\cL$, satisfy $| r_{a}-\langle a_{\cL}, \hat\theta_{\cL}\rangle |\le O(\varepsilon)$ for all $a\in\rows(\Phi)$.
	\end{theorem}
	\begin{proof}
		We first prove the correctness of the algorithm. Suppose for some $(w,\cM,\hat{\theta}_{\cM})\in\cS$, there is $x\in \cR_{\cM}^{w}$ such that $([x_{\cM'}]_{\mathcal{N}^s},\cM',\hat{\theta}_{\cM'})\in\cS$ and $|\langle x_{\cM'},\hat{\theta}_{\cM'}\rangle-\langle w,\hat{\theta}_{\cM}\rangle|>5\varepsilon/2$ and $\cM'\neq \cM$. Consider two cases in Lines 4-7  in Algorithm~\ref{alg2}.
		
		\begin{itemize}
			\item 		Case 1: Suppose $|r_x-\langle w,\hat{\theta}_{\cM}\rangle|\leq 3\varepsilon/2$, then we have that $|r_x - \langle x_{\cM},\hat{\theta}_{\cM}\rangle|\le 2\varepsilon$ and $|r_x - \langle x_{\cM'},\hat{\theta}_{\cM'}\rangle|\ge | \langle x_{\cM'},\hat{\theta}_{\cM'}\rangle-\langle w,\hat{\theta}_{\cM}\rangle| - |r_x - \langle w,\hat{\theta}_{\cM}\rangle| > \varepsilon$. Thus after the iterations, for some $(w,\cM,\hat{\theta}_{\cM})\in S$ and $x\in \cR_{\cM}^{w}$, we have $|r_x - \langle x_{\cM},\hat{\theta}_{\cM}\rangle| \le 2\varepsilon$.  We remove the elements $(\sim,\cM',\hat{\theta}_{\cM'})$ from $\cS$ since there exists an $x\in\rows(\Phi)$ such that $|r_x-\langle x_{\cM'},\hat{\theta}_{\cM'}\rangle|> \varepsilon$.

			\item Case 2:  Assume that $|r_x-\langle w,\hat{\theta}_{\cM}\rangle|> 3\varepsilon/2$ for some $x\in \cR_{\cM}^w$. Then the elements $(\sim,\cM,\hat{\theta}_{\cM})$ get removed from $\cS$ since there exists an $x\in\rows(\Phi)$ such that $|r_x- \langle x_{\cM},\hat{\theta}_{\cM}\rangle| \ge |r_x - \langle w,\hat{\theta}_{\cM}\rangle| -| \langle x_{\cM},\hat{\theta}_{\cM}\rangle-\langle w,\hat{\theta}_{\cM}\rangle| >\varepsilon$. 
		\end{itemize}
		Moreover, Algorithm \ref{alg2} guarantees that
		\begin{itemize}
			\item The elements $(\sim,\cM^*,[{\theta^*}_{\cM^*}]_{\cN^s})$ maintain in the set $\cS$, which involves the ground-truth index set $\cM^*$ and $[{\theta^*}_{\cM^*}]_{\cN^s}\in \cN^s$ such that $|r_x-\langle x_{\cM^*},[{\theta^*}_{\cM^*}]_{\cN^s}\rangle|\leq \varepsilon$. Algorithm \ref{alg2} only eliminates elements $(\sim,\cM,\hat{\theta}_{\cM})$ involving the index set $\cM$ and $\hat{\theta}_{\cM}$ such that $|r_x-\langle x_{\cM},\hat{\theta}_{\cM}\rangle|> \varepsilon$ for some $x\in\rows(\Phi)$. 
			
			\item If no more pairs in the remaining set $\cS$ satisfies the conditions on Line~4 in Algorithm \ref{alg2}, then it must be the case that, for all $(w,\cM,\hat{\theta}_{\cM})\in\cS$ with the remaining set $\cS$ and  $\forall~ x\in \cR_{\cM}^{w}$, $|\langle x_{\cM^*},[{\theta^*}_{\cM^*}]_{\cN^s}\rangle-\langle w,\hat{\theta}_{\cM}\rangle|\leq5\varepsilon/2$, and hence
			\begin{align}
			&	|r_x -\langle w,\hat{\theta}_{\cM}\rangle| = |\langle x, \theta^*\rangle  + \nu_x - \langle w,\hat{\theta}_{\cM}\rangle| \notag\\
			\le& |\langle x_{\cM^* },[{\theta^*}_{\cM^*}]_{\cN^s}\rangle - \langle w,\hat{\theta}_{\cM}\rangle| + \varepsilon\le  7\varepsilon/2,
			\end{align}
			Moreover,
			\[|r_x-\langle x_{\cM}, \hat{\theta}_{\cM}\rangle |\leq      |r_x -\langle w,\hat{\theta}_{\cM}\rangle|  +|\hat{\theta}_{\cM}^{\top}(x_{\cM}-w)|\leq 4\varepsilon.\]

		\end{itemize}
		
		In summary,  for a set $\cL\subseteq[d], |\cL|=s$ and corresponding $\hat{\theta}_{\cL}\in\cN^s$ such that $(\sim,\cL,\hat{\theta}_{\cL})$ in the remaining set $\cS$, we can guarantee that
		\[|r_x-\langle x_{\cL},\hat{\theta}_{\cL}\rangle|\leq 4\varepsilon,\]
		for $\forall~ x\in\rows(\Phi)$.

		We arrive at the sample complexity analysis of the algorithm. If we find $(w,\cM,\hat{\theta}_{\cM})\in\cS$, $\cM'\neq \cM$, $x\in\cR_{\cM}^w$ satisfying the condition on Line 4 in Algorithm \ref{alg2}, we remove either the elements either $(\sim,\cM,\hat{\theta}_{\cM})$ or $(\sim,\cM',\hat{\theta}_{\cM'})$ after querying one action. The loop stops when the condition on Line 4 is not satisfied. Thus, at most $|\cN^s|{d \choose s}$ queries are needed for the algorithm.
		Recall $|\cN^s|$ (\ref{size}), the number of queries in $s$-sparsity case can be bounded by 
		\[ O\left(\left(\frac{1}{\varepsilon}\right)^{s}\cdot {d \choose s}\right).\]
		
	\end{proof}

	When $s$ is a fixed constant, the above theorem demonstrates that $\poly(d/\varepsilon)$-queries are sufficient to learn an $O(\varepsilon)$-optimal action. This is in stark contrast to the $\Omega(\exp(d))$ lower-bound provided in \citet{du2020good} and \citet{lattimore2020learning}. 
	When $s$ is not fixed, the dependence on $\exp(s)$ is undesirable. One may ask, whether it is possible to achieve $\poly(s)$-dependence for some cases, e.g., relaxed error $s^{\delta}\varepsilon$ for some $\delta>0$. 
	Unfortunately,  the next section provides a lower bound that rules out the possibility for $\delta < 1$.
	
	\subsection{Lower bound}\label{lower-bound}
	In this section, we establish an information-theoretical lower bound to show that our upper bound is nearly tight.
	The basic idea
	is by reduction to the INDEX-QUERY problem \citep{du2020good,yao1977probabilistic} using statistical analysis on sub-exponential random variables.
	More formally, it is shown \citep{du2020good} that if one is given a vector of dimension $n$ with only one non-zero entry, then it is necessary to query $\Omega(pn)$ entries of the vector to output the index of the entry with probability $p$. In what follows, we can show that for any algorithm that solves an $s$-sparse $\varepsilon$-misspecified  linear bandit problem, we can use it to solve the INDEX-QUERY problem of size $\Omega(\exp(s))$. The idea is to establish a set of sparse vectors with sub-exponential random variables, such that the vector input to the INDEX-QUERY problem can be embedded into the bandit instance (without any queries to the vector).

	The next lemma is the key tool that will be useful in our lower-bound arguments. 
	It shows that there exists a sparse matrix $\Phi \in \R^{k \times d}$ with sufficiently large $k$ where
	rows have unit norm and sparsity $s$, and all non-equal rows are almost orthogonal.

	\begin{lemma}
		\label{lem:jl}
		For $0<\delta<1$, $c>1$ and $C'=  {\frac{2c^3}{ (1+\tau)\sqrt{c^2-1}}}$ with sufficiently small  $0\leq\tau<1$, 
		\begin{itemize}
			\item 	if $0 < \varepsilon\leq \frac{C's}{ d}$, by choosing $k\geq\sqrt{\delta}\exp\left({\frac{d(1+\tau)\varepsilon^2}{4C'}}\right)$,
			
			\item 	if $\varepsilon> \frac{C's}{ d}$, by choosing $k\geq\sqrt{\delta}\exp\left({\frac{s(1+\tau)\varepsilon}{4}}\right)$,
		\end{itemize}
		there exists a feature matrix $\Phi\in \R^{k\times d}$ with rows such that 
		for all $a, b \in \rows(\Phi)$ with $a \neq b$, $\norm{a}_2= 1$, $\norm{a}_0\leq s$, and $|\langle a,b\rangle| \leq \varepsilon$. 
	\end{lemma}
	\begin{proof}[Proof Sketch]
		The matrix is established by choosing each entry of the matrix $\Phi$ a small probability ($\sim s/d$) to be non-zero and if it is non-zero, the entry follows a Gaussian distribution. The formal proof is provided in Appendix \ref{proof_lemma_1}.
	\end{proof}
	
	As we will show shortly, the matrix in 
	Lemma \ref{lem:jl} can be used to \emph{agnostically} embed an arbitrary INDEX-QUERY problem to a sparse misspecified instance.
	To start with the formal reduction, we introduce the definition of $(\eta,\Delta)$-sound algorithm for linear bandit problem, where the algorithm returns an estimated optimal action $\hat{a}\in\rows(\Phi)$ and an estimation vector $\hat{\theta}\in \R^d$.
	\begin{definition}\label{delta}
		For any $0<\eta<1$ and the approximation error $\Delta \ge \varepsilon$, an algorithm $\mathcal{A}$ solving linear bandit problem is called sound for $(\eta,\Delta)$
		if with probability at least $1-\eta$, algorithm $\cA$ returns the estimated optimal action $\hat{a}$ such that $r_{\hat{a}}\ge \max_xr_{ {x}} -\Delta$.
	\end{definition}
	For any input vector $v$ to the INDEX-QUERY problem (of dimension $k$) with some unknown index $j$ to be non-zero, we can simply take $\Phi$ as the feature matrix, and the $j$-th row of $\Phi$ to be the ground-truth $\theta^*$.
	Then we would have $\|v - \Phi \theta^*\|_{\infty}\le \varepsilon$. Thus any $(\eta,\Delta)$-sound algorithm for some appropriate $\Delta$ would identify the non-zero index in $v$ with good probability and thus inherits the lower bound of INDEX-QUERY.
	The formal lower bound is presented in the following theorem.
	\begin{theorem}\label{prop:badmx} 
		For any $(\eta,\Delta)$-sound linear bandit algorithm $\cA$, there exists an $s$-sparse $\varepsilon$-misspecified linear bandit instance 
		such that algorithm $\cA$ takes
		at least 
		\begin{align}
		(1-\eta)\exp\left(c_0d\cdot \left(\frac{\varepsilon}{\Delta}\right)^2\right),&~\text{if }~0 < \frac{\varepsilon}{\Delta}\leq \frac{C's}{ d},\\
		(1-\eta)\exp\left({\frac{c_1s(1+\tau)\varepsilon}{\Delta }}\right),&~\text{if }~\frac{\varepsilon}{\Delta}> \frac{C's}{ d},
		\end{align}
		actions to halt, where $c_0, c_1, C'$ are absolute constants.
	\end{theorem}
	
	\begin{proof}
		We begin with the construction of the hard $s$-sparsity instances. 
		Consider an INDEX-QUERY problem with dimension $k$. Suppose the input vector with the $i^*$-index (unknown to the algorithm) is non-zero, i.e., $e_{i^*}$. Here, $e_i$ is the standard unit vector with the $i$-th coordinate equaling $1$. 
		In our hard instance, we choose reward $r_{x}=2\Delta$ when $x=a_{i^*}$ with $i^*\in[k]$, otherwise is $0$. 
		Now we show that there exists a linear feature representation that approximates the reward vector $\Delta e_{i^*} \in\R^k$ with a uniform error. Based on Lemma \ref{lem:jl}, let $\Phi$ be the matrix $\rows(\Phi) = (a_i)_{i=1}^k$ such that for all $a_i, a_j \in \rows(\Phi)$ with $i \neq j$, $\norm{a_i}_2= 1$ and $|\langle a_i,a_j\rangle| \leq \varepsilon/(2\Delta) $. With $\theta^* =2\Delta a_{i^*}$, we have $\Phi \theta^* =(  2\Delta a_1^\top a_{i^*}, \dots,   2\Delta a_{i^*}^\top a_{i^*}, \dots,  2\Delta  a_k^\top a_{i^*})^\top$. 
		By choice of $\Phi$, the ${i^*}$-th component of $\Phi\theta^*$ is $\Delta$ and the others are all less than $\varepsilon$ in absolute value. Hence, we can represent the reward vector $2\Delta e_{i^*} $ by
		$2\Delta e_{i^*} = \Phi \theta^* + \nu $ for some $\nu\in [-\varepsilon,\varepsilon]^k$. 
		
		
		Then an $(\eta, \Delta)$-sound algorithm would identify an action $a$, such that with probability at least $1-\eta$, $a^\top \theta^* \ge 2\Delta - \Delta = \Delta$, which is only possible if $a = a_{i^*}$. Hence the algorithm would output $i^*$ with probability at least $1-\eta$. By the lower bound of the INDEX query problem (e.g., Theorem~A1 in \citep{du2020good}), the algorithm takes at least $\Omega((1-\eta)k)$ queries in the worst-case.
		
		In the construction, we only need Lemma~\ref{lem:jl} to hold for $k$ with the correct parameters. Hence we have 
		\begin{itemize}
			\item 	if $0 < \varepsilon\leq \frac{C's}{ d}$, then  $k\geq\sqrt{\delta}\exp\left({\frac{d(1+\tau)\varepsilon^2}{16C'\Delta^2}}\right)$, and
			
			\item 	if $\varepsilon> \frac{C's}{ d}$, then $k\geq\sqrt{\delta}\exp\left({\frac{s(1+\tau)\varepsilon}{8\Delta}}\right)$,
		\end{itemize}
		for constant $\tau$, $\delta$, and $C'$,
		completing the proof.

	\end{proof}

	\begin{remark}
		The above theorem shows that even if we relax the approximation error to $s^\delta \varepsilon$ for some $0<\delta<1$, the $\exp(s)$ dependence is unavoidable. Hence our upper bound in the previous section is nearly tight.
		However, this lower bound does not rule out the improvement in terms of $\varepsilon^{-s}$ and  efficient regimes when $\Delta =\Omega(s^{\delta} \varepsilon)$ for some $\delta\geq 1$.
		We will explore both settings in the rest of the paper.
		
	\end{remark}

	\section{Improvement on the $\varepsilon^{-s}$ Dependence}\label{improve}
	Even though the dependence of $d^s$ is unavoidable, we can improve the upper bound in Theorem \ref{theo2} by
	eluding the dependence of $\varepsilon$. The fundamental idea of the improved algorithm is based on a mix of G-optimal design and elimination argument. Instead of guessing an estimator $\hat{\theta}$ for $\theta^*$, we use G-optimal design to estimate $\hat{\theta}$ concerning an index set $\cM\subset [d]$. Then for each estimator $\hat{\theta}$ and $\cM$, we check the actions that have similar features restricting to $\cM$. 
	The rest of the elimination argument is similar to Section~\ref{m-sparse}.
	Yet the optimal G-optimal design only gives an error guarantee of $O(\varepsilon\sqrt{s})$, which worsens our error guarantee.
	Below, we present the algorithm more formally.
	
	We start with an essential theorem in G-optimal design which shows that there exists a near-optimal design with a small core set.
	
	\begin{theorem}[\citet{todd2016minimum}]\label{thm:todd}
		Given a matrix $A\in\R^{k\times s}$ and a probability distribution $\rho : \rows(A) \to [0,1]$, let $G(\rho) \in \R^{s\times s}$\footnote{Without loss of generality, we assume $G(\rho)$ is invertible in the rest of the paper. If not, we can discard columns in $\Phi$ until the $\Phi$ is full column rank.} and $g(\rho) \in \R$ be given by
		\begin{align*}
		G(\rho) &= \sum_{a \in \rows(A)} \rho(a) a a^\top\,, & 
		g(\rho) &= \max_{a\in \rows(A)} \norm{a}_{G(\rho)^{-1}}^2 \,.
		\end{align*}
		There exists a probability distribution $\rho$ such that $g(\rho) \leq 2m$ and the size of the support of $\rho$ is at most $4s \log \log(s) + 16$.
	\end{theorem}
	
	\begin{remark}
		The distribution satisfying the results in Theorem \ref{thm:todd} can be computed by the Frank-Wolfe algorithm introduced in \citep{todd2016minimum}[Chapter 3] after $O(ks^2)$ computations.    
	\end{remark}

	Let $\cS\subset [d]^s$ be all the subsets of cardinality $s$.
	For each $\cM\in\cS$, suppose that $\rho_{\cM}$ is a probability distribution over $\rows(\Phi_{\cM})$ satisfying the results of Theorem \ref{thm:todd}, where $\Phi_{\cM}\in\R^{k\times s}$ is the sub-matrix of $\Phi\in\R^{k\times d}$. In the following, we use $G_{\cM}(\rho_{\cM})$ to present $G(\rho)$ defined in Theorem \ref{thm:todd} with respect to $\cM$.  We begin with querying actions to estimate $\hat{\theta}_{\cM}$ based on the support of $\rho_{\cM}$ and obtain rewards:
	\begin{align}\label{estimate}
	\hat{\theta}_{\cM} = G_{\cM}(\rho_{\cM})^{-1}\sum_{a\in\rows(\Phi_{\cM}), \rho_{\cM}(a)\neq0}
	\rho_{\cM}(a)r_a a,	\end{align}
	
	With Theorem \ref{thm:todd}, we can show that, for all $b\in\rows(\Phi)$ and $ \lceil 4s\log\log (s)+16\rceil$ queries, we have
	\begin{align}\label{error}
	|\langle b_{\cM^*},\hat{\theta}_{\cM^*}\rangle-\langle b,\theta^*\rangle|\leq \varepsilon\sqrt{2s},
	\end{align}
	where $b_{\cM^*}\in\RR^s$ is the sub-vector of $b\in\RR^{d}$.
	For $\cM,\cM'\in\cS$, we try to find some $x\in \rows(\Phi)$ making $\hat{\theta}_{\cM'}^{\top}x_{\cM'}$ deviate from $\hat{\theta}_{\cM}^{\top}x_{\cM}$. We query such $x$ and receive the corresponding reward $r_x$. By comparing the difference between $r_x$ and $\hat{\theta}_{\cM}^{\top}x_{\cM},\hat{\theta}_{\cM'}^{\top}x_{\cM'}$, we can know whether the subset $\cM$ or $\cM'$ of $x$ is more likely to determine the reward $r_x$, and hence eliminate the incorrect parameter-set. The full algorithm is presented in Algorithm~\ref{alg3}.

	\begin{algorithm}[htb!]
		\caption{$(\varepsilon^{-s})$-Free Algorithm }
		\begin{algorithmic}[1]\label{alg3}
			\STATE
			\textbf{Input:} feature matrix $\Phi\in\RR^{k\times d}$.
			\STATE
			\textbf{Initialize:} $\mathcal{S}:=  \{\cM: \cM\subseteq[d], |\cM|=s\}$.
			\STATE For each $\cM\in\cS$, estimate $\hat{\theta}_{\cM}$ based on (\ref{estimate}). \label{est}
			
			\WHILE{ there exit $\cM,\cM'\in\cS$, $\cM\neq\cM'$, and $x\in \rows(\Phi)$ such that $|\langle x_{\cM'},\hat{\theta}_{\cM'}\rangle-\langle x_{\cM},\hat{\theta}_{\cM}\rangle|>2\varepsilon(1+\sqrt{2s})$}\label{con3}
			\STATE Query the action $x$ and receive a reward $r_x = \langle x, \theta^*\rangle + \nu_x$ where $\nu_x\in[-\varepsilon,\varepsilon]$.
			\STATE If $|r_x-\langle x_{\cM},\hat{\theta}_{\cM}\rangle|\leq \varepsilon(1+\sqrt{2s})$ then $\cS = \cS \backslash{\cM'}$.
			\STATE Otherwise $\cS = \cS \backslash{\cM}$, if $|r_x-\langle x_{\cM'},\hat{\theta}_{\cM'}\rangle|> \varepsilon(1+\sqrt{2s})$ then $\cS = \cS \backslash{\cM'}$.
			\ENDWHILE\label{end_con3}

			\STATE  Find a certain set $\cL\subseteq[d], |\cL|=s$ such that $\cL\in\cS$ and estimation $\hat{\theta}_{\cL}\in\R^s$.
			
			\STATE \textbf{Output:} $\hat \theta_{\cL}$ and $\cL$
			
		\end{algorithmic}
	\end{algorithm}

	\begin{theorem}\label{theo3}
		After \[O\left(s\log s\cdot{d \choose s}\right)\] number of queries, the outputs of Algorithm \ref{alg3}, $\hat \theta_{\cL}$ and $\cL$, satisfy $| r_{a}-\langle a_{\cL}, \hat\theta_{\cL}\rangle |\le O(\varepsilon\sqrt{s})$ for all $a\in\rows(\Phi)$.
	\end{theorem}
	\begin{proof}
		We first prove the correctness of the algorithm. Suppose we find some $\cM,\cM'\in\cS$, $\cM\neq\cM'$, and $x\in \rows(\Phi)$ $|\langle x_{\cM'},\hat{\theta}_{\cM'}\rangle-\langle x_{\cM},\hat{\theta}_{\cM}\rangle|>2\varepsilon(1+\sqrt{2s})$. Consider two cases in Lines 4-8  in Algorithm~\ref{alg3}.
		
		\begin{itemize}
			\item 		Case 1: Suppose we have $|r_x - \langle x_{\cM},\hat{\theta}_{\cM}\rangle| \le \varepsilon(1+\sqrt{2s})$.  We remove the element $\cM'$ from $\cS$ since there exists an $x\in\rows(\Phi)$ such that $|r_x - \langle x_{\cM'},\hat{\theta}_{\cM'}\rangle|\ge | \langle x_{\cM'},\hat{\theta}_{\cM'}\rangle-\langle x_{\cM},\hat{\theta}_{\cM}\rangle| - |r_x - \langle x_{\cM},\hat{\theta}_{\cM}\rangle| > \varepsilon(1+\sqrt{2s})$.

			\item Case 2:  Assume that $|r_x - \langle x_{\cM},\hat{\theta}_{\cM}\rangle| > \varepsilon(1+\sqrt{2s})$, then the element $\cM$ gets removed from $\cS$. We can also remove the other index set $\cM'$ from $\cS$ if $|r_x-\langle x_{\cM'},\hat{\theta}_{\cM'}\rangle|> \varepsilon(1+\sqrt{2s})$.
		\end{itemize}
		Moreover, Algorithm \ref{alg3} guarantees that
		\begin{itemize}
			\item The ground-truth index set  $\cM^*$ maintains in the set $\cS$. According to (\ref{error}), for all $x\in\rows(\Phi)$, we have $|r_x-\langle x_{\cM^*},\hat{\theta}_{\cM^*}\rangle|\leq \varepsilon(1+\sqrt{2s})$. Algorithm \ref{alg3} only eliminates $\cM$ such that $|r_x-\langle x_{\cM},\hat{\theta}_{\cM}\rangle|> \varepsilon(1+\sqrt{2s})$ for some $x\in\rows(\Phi)$. After each query, Algorithm \ref{alg3} removes at least one element from $\cS$.
			
			\item If no more pair in the remaining set $\cS$ satisfies the conditions on Line~\ref{con3} in Algorithm \ref{alg3}, then it must be the case that, for all $\cM\in\cS$ with the remaining set $\cS$ and  $\forall~ x\in \rows(\Phi)$, $|\langle x_{\cM^*},\hat{\theta}_{\cM^*}\rangle-\langle x_{\cM},\hat{\theta}_{\cM}\rangle|\leq2\varepsilon(1+\sqrt{2s})$. According to (\ref{error}), we have
			\begin{align}
			&|r_x-\langle x_{\cM}, \hat{\theta}_{\cM}\rangle |\notag\\
			\leq    &  |r_x -\langle x_{\cM^*},\hat{\theta}_{\cM^*}\rangle|  +|\langle x_{\cM^*},\hat{\theta}_{\cM^*}\rangle-\langle x_{\cM},\hat{\theta}_{\cM}\rangle|\notag\\
			\leq & 3\varepsilon(1+\sqrt{2s}).
			\end{align}

		\end{itemize}
		
		In summary,  for a set $\cL\subseteq[d], |\cL|=s$ in the remaining set $\cS$ and the estimation $\hat{\theta}_{\cL}\in\RR^s$, we can guarantee that
		\[|r_x-\langle x_{\cL},\hat{\theta}_{\cL}\rangle|\leq 3\varepsilon(1+\sqrt{2s}),\]
		for $\forall~ x\in\rows(\Phi)$.

		We arrive at the sample complexity analysis of the algorithm. The estimation on Line \ref{est} in Algorithm \ref{alg3} takes $\lceil 4s\log\log (s)+16\rceil{d \choose s}$ queries. If we find $\cM,\cM'\in\cS$, $\cM\neq\cM'$, and $x\in\rows(\Phi)$ satisfying the condition on Line \ref{con3} in Algorithm \ref{alg3}, we remove at least one element from $\cM,\cM'$ after querying one action. The loop stops when the condition on Line 4 is not satisfied. Thus, the number of queries in the $s$-sparsity case can be bounded by 
		\[ O\left(s\log s\cdot {d \choose s}\right).\]
		
	\end{proof}

\label{key}	\section{A $\poly(s)$-Query Algorithm for Benign Features}\label{sec:relax}
	The lower bound derived in Section \ref{lower-bound} does not rule out the possibility of $\exp(s)$-free bound when $\Delta = O(s^{\delta}\varepsilon )$  for $\delta\geq 1$, which we attempt to achieve in this section. The core idea of our algorithm is based on feature compression followed by action-elimination bandit learning. Specifically, we compress the feature vectors and the sparse parameter vector to a lower dimensional vector space, thus converting the sparse linear bandits to a dense case with much lower dimensional features. Note that this compression is agnostic to the ground-truth parameters. Then we implement action-elimination learning in compressed linear bandits. The detailed algorithm is provided in the following.

	We here consider the finite setting where the number of rows, $k$, in the feature matrix $\Phi$ is finite (recall the definition in (\ref{def:sparse_linear})). This argument is without loss of generality as we can always find an $\varepsilon$-net to cover the actions if there are infinitely many.
	By Johnson-Linderstrauss lemma \citep{johnson1984extensions}, we have that for some $p=\Theta(\log (k)/\upsilon^2)$, there is a function $f: \R^d\rightarrow \R^p$ that preserves inner product, i.e., for each $a\in\rows(\Phi)$,
	\begin{align}\label{fun}
	\langle f(a), f(\theta^*)\rangle  = 	\langle a, \theta^*\rangle \pm 2\upsilon,
	\end{align}
	for some error $\upsilon>0$. 
	Such a function can be found efficiently using techniques in, e.g., \citet{kane2014sparser}.
	Hence, we transform the previous sparse linear model $\langle a, \theta^*\rangle$ where $a,\theta^*\in\R^s$ to a new linear model $\langle f(a), f(\theta^*)\rangle$ where  $f(a),f(\theta^*)\in \R^p$ with $p<d$. We apply G-optimal design mentioned in (\ref{estimate}) to get an estimation of $f(\theta^*)$, i.e., $\hat{\theta}_f$. The detailed algorithm is illustrated in Algorithm \ref{alg4} where $C>0$ is a suitable large constant.

	\begin{algorithm}[htb!]
		\caption{$\poly(s)$-Query Algorithm for Benign Features}
		\begin{algorithmic}[1]\label{alg4}
			\STATE
			\textbf{Input:} feature matrix $\Phi\in\RR^{k\times d}$, function $f:\R^d\rightarrow \R^p$ (\ref{fun}), the total time steps $n$.
			\STATE
			\textbf{Initialize:} $\mathcal{S}:= \{f(a)\in\R^p: a\in \rows(\Phi)\}$.
			\WHILE{number of queries is no greater than $n$}
			
			\STATE Compute the probability distribution $\rho : \cS \to [0,1]$ satisfying the results of Theorem \ref{thm:todd}.
			\STATE Compute $\hat{\theta}_f=G(\rho)^{-1}\sum_{a\in\cS}
			\rho(a)r_a a$ where the reward $r_a$ is received by querying the action $a$.
			
			\STATE Update active action set:
			\[			\small\cS \leftarrow \bigg\{a \in \cS : \max_{b \in \cS} \ip{\hat{\theta}_f, b - a} 
			\leq C\cdot(\log(k))^{\frac{1}{4}}\sqrt{\varepsilon}\bigg\}.\]

			\ENDWHILE
			\STATE \textbf{Output:} $\hat{\theta}_f\in\RR^p$
		\end{algorithmic}
	\end{algorithm}
	\begin{theorem}\label{theo5}
		Suppose there is a function $f: \R^d\rightarrow \R^p$ satisfied (\ref{fun}). 
		After $n\ge O\left(\sqrt{\log k}/\varepsilon\right)$ number of queries, the output of Algorithm \ref{alg4} satisfies $| r_{a}-\langle f(a), \hat\theta_f\rangle |\le  O( (\log(k))^{1/4} \sqrt{\varepsilon}+\varepsilon)$ for all $a\in\rows(\Phi)$. 
	\end{theorem}
	
	\begin{proof}[Proof Sketch]		
		We start with the approximation error of $f(\theta^*) $.
		
		Similar to the G-optimal design in Section \ref{improve},  we have
		\begin{align}\label{error1}
		&|\langle f(a),\hat{\theta}_f\rangle-\langle a,\theta^*\rangle|\notag\\
		\leq&
		|\langle f(a),\hat{\theta}_f\rangle-\langle f(a),f(\theta^*)\rangle|+|\langle f(a),f({\theta^*})\rangle-\langle a,\theta^*\rangle|,
		\end{align}
		for $\forall ~a\in\rows(\Phi)$.
		
		The first term in (\ref{error1}) can be termed as misspecified linear bandits in $\R^p$.
		Similar to (\ref{error}), the first term in (\ref{error1}) can be bounded by
		\begin{align}
		|\langle f(a),\hat{\theta}_f\rangle-\langle f(a),f(\theta^*)\rangle|
		\leq C\left(\varepsilon\sqrt{p} \right)
		\end{align}
		with $O(p\log(p))$ number of queries, where $C> 0$ is a suitably universal large constant. The second term in (\ref{error1}) can be bounded by $2\upsilon$. Hence, we have
		\begin{align}\label{error2}
		|\langle f(a),\hat{\theta}_f\rangle-\langle a,\theta^*\rangle|
		\leq  C\left(\varepsilon\sqrt{p}+\upsilon\right),
		\end{align}
		Recall that $p=\Theta(\log (k)/\upsilon^2)$, thus $C(\varepsilon\sqrt{p}+\upsilon)$ can be presented as an function with respect to $\upsilon$ ($\upsilon> 0$), given by 
		\[g(\upsilon)=C(\varepsilon\sqrt{\log (k)/\upsilon^2}+\upsilon).\]
		
		By optimizing $g(\upsilon)$ with respect to $\upsilon$, we have the approximation error of $O((\log(k))^{\frac{1}{4}}\sqrt{\varepsilon})$ achieved by the number of queries $O(\sqrt{\log k}/\varepsilon)$. 
		
		We can derive the final approximate error as 
		\begin{align}
		|\langle f(a),\hat{\theta}_f\rangle-\langle a,\theta^*\rangle|
		\leq  C\left((\log(k))^{\frac{1}{4}}\sqrt{\varepsilon} \right).
		\end{align}

	\end{proof}
	
	\begin{corollary}\label{coro}
		Based on the notations in Theorem \ref{theo5}, if setting $\upsilon = O(s^{\delta}\varepsilon )$  for $\delta\geq 1$, the number of queries $O(s^{1+\delta})$ can be achieved whenever $\log(k)\leq \varepsilon^2s^{2(1+\delta)}$. Additionally, the regret of Algorithm \ref{alg4} is bounded by $O(s^{\delta}\varepsilon n\log(n))$.
		
	\end{corollary}
	
	According to Corollary \ref{coro}, when the coefficient $s^{\delta}<\sqrt{d}$, Algorithm \ref{alg4} can break the $\varepsilon\sqrt{d}$ barrier with polynomial queries in all parameters if $\log(k)$ is small, which is achievable if the feature space possesses certain benign structures. For instance, the features are (close to) sparse as the instance in our lower bound construction. This also demonstrates that this result may not admit additional improvement as it resolves the lower bound instance.
	
	All results above focus on the noiseless case. We further give a discussion on the noisy cases in Section \ref{noise}.

	\section{A $\poly (s)$-Query Algorithm for General Features}
	Theorem \ref{theo5} presents an algorithm with sample complexity dependent on $\log(k)$ where $k$ is the number of actions. Corollary \ref{coro} shows that it is possible to achieve sample complexity of $\poly(s)$ when $k$ satisfies the condition $\log(k) \leq \varepsilon^2 s^{2(1+\delta)}$ for $\delta\geq 1$. However, to accommodate a wider range of scenarios, we aim for a sample complexity with a better dependence. 
	In the following section, we will describe the method for achieving a sample complexity dependent on $\poly(s)$ for general features leveraging the ideas in Section~\ref{sec:relax}.
	
	The core idea of our algorithm is to select a submatrix $\Psi\in\RR^{k' \times d}$ from the feature matrix $\Phi\in\RR^{k\times d}$ where $k'<k$. The submatrix $\Psi$ should contain enough representative actions, which we obtain by using G-optimal design concerning all possible $s$-subset $\cM'\subset[d]$ as (\ref{estimate}) and collecting the corresponding actions $a\in\RR^d$. Then, we apply the same compression process as Section \ref{sec:relax} to reduce the dimensionality of the feature matrix $\Psi$ and the sparse parameter vector $\theta^*$. Finally, we estimate the $s$-sparsity parameter $\theta^*$ based on the above information. This method consists of three main steps:
	\begin{enumerate}
		\item  \textbf{G-optimal design:} 
		For each $\cM'\subset[d]$ such that $|\cM'| =s$, we first find a probability distribution $\rho_{\cM'}$ over $\rows(\Phi_{\cM'})$ that meets the conditions of Theorem \ref{thm:todd}. We then use this distribution $\rho_{\cM'}$ to generate $z:=4s\log\log (s)+16$ distinct feature vectors $a_{\cM'}\in\RR^{s}$. We collect all the corresponding actions $a\in\RR^d$ and denote them as $\Psi\in\R^{{d\choose s}\cdot z\times d}$.
		\item \textbf{Compression:} By Johnson-Linderstrauss lemma \citep{johnson1984extensions}, we have that for some $q=\Theta(\log ({d\choose s}\cdot z)/\varphi^2)$, there is a function $h: \R^d\rightarrow \R^q$ that preserves inner product, i.e., for each $a\in\rows(\Psi)$ there is
		\begin{align}\label{fun1}
		\langle h(a), h(\theta^*)\rangle  = 	\langle a, \theta^*\rangle \pm 2\varphi,
		\end{align}
		for some error $\varphi>0$.
		\item \textbf{Estimation:} After inputting Algorithm \ref{alg4} with the feature matrix $\Psi$, function $h$ and the total time steps $z$, we can estimate the compressed parameter $\hat{\theta}_{h}\in\RR^{q}$. Based on $\hat{\theta}_{h}\in\RR^{q}$ and $\Psi_h\in\R^{{d\choose s}\cdot z\times q}$ whose rows are $h(a)$ for all $a\in\rows(\Psi)$, the $s$-sparsity estimator $\hat{\theta}\in\RR^d$ can be computed via the convex optimization problem (\ref{sparselr}).

	\end{enumerate}
	Using the above notations, the proposed algorithm is illustrated in Algorithm \ref{alg5}. The following theorem presents the sample complexity of our method.
	\begin{algorithm}[htb!]
		\caption{$\poly (s)$-Query Algorithm for General Features}
		\begin{algorithmic}[1]\label{alg5}
			\STATE
			\textbf{Input:} feature matrix $\Phi\in\RR^{k\times d}$.
			\STATE
			Compute $\rho_{\cM'} : \rows(\Phi_{\cM'}) \to [0,1]$ satisfying the results of Theorem \ref{thm:todd} for each sub-index-set $\cM'\subseteq[d]$ with $|\cM'| =s$.
			\STATE Based on $\{\rho_{\cM'}\}$, we collect representative action features as $\Psi\in\R^{{d\choose s}\cdot z\times d}$ where $z:=4s\log\log (s)+16$.		
			\STATE Find a function $h:\R^d\rightarrow \R^q$ satisfying (\ref{fun1}).
			\STATE Get $\hat{\theta}_{h}\in\RR^{q}$ and $\Psi_h\in\R^{{d\choose s}\cdot z\times q}$ via inputting ($\Psi$,$h$,$z$) to Algorithm \ref{alg4}.
			\STATE Estimate the parameter $\hat{\theta}\in\RR^d$ via the convex optimization problem
			\begin{align}\label{sparselr}
			\min_{\theta\in\RR^d} \norm{\Psi\theta-\Psi_h\hat{\theta}_{h}}_{\infty}~\text{subject to }\norm{\theta}_0\leq s.
			\end{align}
			\STATE
			\textbf{Output:} $\hat{\theta}\in\RR^{d}$

		\end{algorithmic}
	\end{algorithm}
	
	\begin{theorem}\label{theo6}
		Suppose there is a function $h: \R^d\rightarrow \R^q$ satisfied (\ref{fun1}) for $q=\Theta(\log ({d\choose s}\cdot z)/\varphi^2)$ and $\varphi=(s\log(d))^{1/4}\sqrt{\varepsilon}$. 
		Then after $n\ge O(\sqrt{s\log (d)}/\varepsilon)$ number of queries, the output of Algorithm \ref{alg5} satisfies $| r_{a}-\langle a, \hat\theta\rangle |\le  O( (s\log(d))^{1/4} \sqrt{s\varepsilon}+\varepsilon)$ for all $a\in\rows(\Phi)$.

	\end{theorem}
	\begin{proof}[Proof Sketch]		
		We start with the approximation error of $\theta^*\in\RR^d$ based on the feature matrix $\Psi\in\RR^{k'\times d}$. According to (\ref{error2}) in Section \ref{sec:relax}, we can apply G-optimal design to compute the estimator $\hat{\theta}_h\in\RR^q$ satisfying \begin{align}\label{error3}
		|\langle h(a),\hat{\theta}_h\rangle-\langle a,\theta^*\rangle|
		\leq  C'\left(\varepsilon\sqrt{q}+\varphi\right),
		\end{align}
		where $C'>0$.
		Recall that $q=\Theta(\log ({d\choose s}\cdot z)/\varphi^2)$ with $z:=4s\log\log (s)+16$, thus $C'(\varepsilon\sqrt{q}+\varphi)$ can be presented as an function with respect to $\varphi$ ($\varphi> 0$), given by 
		\[g(\varphi)=C''(\varepsilon\sqrt{s\log (d)/\varphi^2}+\varphi),\]
		where $C''>0$.
		The minimum of $g(\varphi)$ achieves when $\varphi^* =(s\log(d))^{1/4}\sqrt{\varepsilon}$ such that  
		\[g(\varphi^*)=C\left((s\log(d))^{1/4} \sqrt{\varepsilon}\right),\]where $C>0$.
		Hence, we have the approximation error, $\forall~ b\in\rows({\Psi})$ 
		\begin{align}\label{hb}
		|\langle h(b),\hat{\theta}_h\rangle-\langle b,\theta^*\rangle|
		\leq  C\left((s\log(d))^{1/4} \sqrt{\varepsilon}\right),
		\end{align}
		which is achieved by the number of queries
		\[O\left({\log\left( {d\choose s}\cdot z\right)}/(v^*)^2\right)=O\left(\sqrt{s\log(d)}/\varepsilon\right).\] 
		For $a\in\rows(\Psi)$, the approximate error of estimator $\hat{\theta}$ in (\ref{sparselr}) can be bounded by
		\begin{align}\label{atheta}
		\left|\ip{a,\hat\theta}-\ip{a,\theta^*}\right|&\leq \left|\ip{a,\hat\theta}-\ip{h(a),\hat\theta_h}\right|+\left|\ip{h(a),\hat\theta_h}-\ip{a,\theta^*}\right|\notag\\
		&\overset{(a)}{\leq} \max_{a\in\rows(\Psi)}\left|\ip{a,\hat\theta}-\ip{h(a),\hat\theta_h}\right|+C\left((s\log(d))^{1/4} \sqrt{\varepsilon}\right)\notag\\
		&\overset{(b)}{\leq} \max_{a\in\rows(\Psi)}\left|\ip{a,\theta^*}-\ip{h(a),\hat\theta_h}\right|+C\left((s\log(d))^{1/4} \sqrt{\varepsilon}\right)\notag\\
		&=O\left((s\log(d))^{1/4} \sqrt{\varepsilon}\right),
		\end{align}
		where step (a) and the last step come from (\ref{hb}) and step (b) derives due to the estimator $\hat{\theta}$ of the problem (\ref{sparselr}). 
		
		From the estimator $\hat{\theta}$, we can get an index set $\hat{\cM}:=\{i~|~\hat\theta_i\neq 0,~i\in[d]\}$. 
		Let ${\cM'} = \hat{\cM}\cup \cM^*$, we then have\footnote{Same as Theorem \ref{thm:todd}, we assume $G_{\cM'}$ is invertible. If not, we can discard columns in $\Phi$ until the $\Phi_{\cM'}$ is full column rank.}
		\begin{align}\label{thetam}
		\hat{\theta}_{\cM'} =
		G_{\cM'}^{-1}G_{\cM'}\hat{\theta}_{\cM'}
		= G_{\cM'}^{-1}\sum_{a\sim\rho_{\cM'}}
		\ip{a_{\cM'},\hat{\theta}_{\cM'}}a_{\cM'},
		\end{align}
		which helps us to bound the approximate error of $\theta^*$ for $\forall~b\in\rows(\Phi)$. Thus, there is
		\begin{align}
		\ip{b ,\hat{\theta} }-\ip{b,\theta^*} &= \ip{b_{\cM'} ,\hat{\theta}_{\cM'} -\theta_{\cM'}^*} \notag
		\\
		& \overset{(a)}{=}\left\langle{b_{\cM'} ,	 G_{\cM'}^{-1}\sum_{a\sim\rho_{\cM'}}
			\ip{a_{\cM'},\hat{\theta}_{\cM'}}a_{\cM'} -\theta_{\cM'}^*} \right\rangle	\notag\\
		& =\left\langle{b_{\cM'} ,	 G_{\cM'}^{-1}\sum_{a\sim\rho_{\cM'} }
			|\ip{a_{\cM'},\hat{\theta}_{\cM'}} -\ip{a_{\cM'},\theta_{\cM'}^*}    } |a_{\cM'}\right\rangle	\notag\\
		&= \sum_{a\sim\rho_{\cM'}}  |\ip{a_{\cM'},\hat{\theta}_{\cM'}} -\ip{a_{\cM'},\theta_{\cM'}^*} |\ip{b_{\cM'},  G_{\cM'}^{-1} a_{\cM'}} \nonumber \\
		&\overset{(b)}{\leq }C\left((s\log(d))^{1/4} \sqrt{\varepsilon}\right) \sum_{a\sim\rho_{\cM'} }   |\ip{b_{\cM'},  G_{\cM'}^{-1} a_{\cM'}}| \nonumber \\ 
		&\overset{(c)}{\leq }C\left((s\log(d))^{1/4} \sqrt{\varepsilon}\right) \sqrt{\sum_{a\sim\rho_{\cM'} }   \ip{b_{\cM'},  G_{\cM'}^{-1} a_{\cM'}}^2} \nonumber \\
		&= C\left((s\log(d))^{1/4} \sqrt{\varepsilon}\right) \sqrt{\norm{b_{\cM'}}^2_{ G_{\cM'}^{-1}}} \notag\\
		&
		\overset{(d)}{\leq }C\left((s\log(d))^{1/4} \sqrt{\varepsilon}\right) \sqrt{g_{\cM'}(\rho_{\cM'})} \notag\\
		&=O\left( (s\log(d))^{1/4} \sqrt{s\varepsilon}\right),
		\end{align}
		where step (a) comes from the definition of $G_{\cM'}$ (\ref{thetam}), step (b) derives from Holder’s inequality and the inequality (\ref{atheta}), step (c) is due to Cauchy–Schwarz inequality, and step (d) follows the definition of $g_{\cM'}(\rho_{\cM'})$ in Theorem \ref{thm:todd}.
		Hence, we ensure that $|r_{ {b}}-\ip{b_{\cM'} ,\hat{\theta}_{\cM'} }|\leq O((s\log(d))^{1/4} \sqrt{s\varepsilon}+\varepsilon)$ for all $b\in\rows(\Phi)$.
	\end{proof}	
	
	The results in Theorem \ref{theo6} do not depend on the number of actions $k$, unlike Theorem \ref{theo5}. This is achieved by selecting representative actions and applying compression to get the submatrix $\Psi_h$, followed by estimation based on the submatrix. In other words, this method works for general features, not just benign ones introduced in Section \ref{sec:relax}. The following corollary restates Theorem \ref{theo5}. It shows the relaxed requirements on the sparse linear bandit model to achieve $O(s\varepsilon)$-optimal actions within $O(s)$ queries, which present more general results compared to Corollary \ref{coro}.
	\begin{corollary}
		Based on the notations in Theorem \ref{theo6}, if $s=\Omega (\log(d)/\varepsilon^{2})$, $O(s{\varepsilon}  )$-optimal actions can be achieved with the number of queries $O(s)$.
	\end{corollary}
	
	\section{Conclusions}
	We aim to utilize the sparsity in linear bandits to remove the $\varepsilon\sqrt{d}$ barrier in the approximation error in existing results \citep{lattimore2020learning} about the misspecified setting. We provide a thorough investigation of how sparsity helps in learning misspecified linear bandits.
	
	We establish novel algorithms that obtain $O(\varepsilon)$-optimal actions by querying ${O}(\varepsilon^{-s}d^s)$ actions, where $s$ is the sparsity parameter. 
	For fixed sparsity $s$, the algorithm finds an $O(\varepsilon)$-optimal action with $\poly(d/\varepsilon)$ queries, removing the dependence of $O(\varepsilon\sqrt{d})$.
	The $\varepsilon^{-s}$ dependence in the sample bound can be further improved to $\tilde{O}(s)$ if we instead find an $O(\varepsilon\sqrt{s})$ suboptimal actions.
	We establish information-theoretical lower bounds to show that our upper bounds are nearly tight. In particular, we show that any algorithms that can obtain $O(\Delta)$-optimal actions need to  query ${\Omega}(\exp(s\varepsilon/\Delta))$ samples from the bandit environment. 
	We further break the $\exp(s)$ sample barrier by showing an algorithm that achieves $O(s\varepsilon)$ sub-optimal actions while only querying $\poly(s/\varepsilon)$ samples. 
	
	Starting from our results on the general bound in misspecified sparse linear bandits, it is interesting to explore results in different bandit learning settings, e.g., contextual bandit problems, RL problems, and distributed/federated-learning settings. 

    \section{Acknowledgement}
    This work is supported by NSF grant NSF-2221871  and DARPA grant HR00112190130.
    We would like to thank Tor Lattimore for proposing the problem during LY's visit to Deepmind, and for the insightful discussions with Tor Lattimore and Botao Hao.
 
	\bibliography{Reference}
	\bibliographystyle{icml2023}

	\newpage
	\appendix
	\onecolumn
	
	\section{Proof of Lemma \ref{lem:jl}}\label{proof_lemma_1}
	
	Let $\cP = \{a_1,a_2,\cdots,a_k\}$ be a set of $k$ independent random vectors in $\RR^d$. For $\forall i \in [k]$, $a_i = [a_{i1}, a_{i2},\cdots, a_{id}]^{\top}\in\RR^d$, we have
	
	\begin{align}\label{a_il}
	a_{i\ell} := \begin{cases} a_{i\ell} \sim\cN(0,\frac{1}{s})& \text {with probability } \frac{s}{d},\\ 0 & \text { otherwise}.\end{cases}
	\end{align}
	Thus, we have the following properties
	\begin{align*}
	\EE[\langle a_i,a_i\rangle] &=1,~ \forall i \in[k],\\
	\EE[\langle a_i,a_j\rangle] & = 0,~ \forall i,j \in[k], i\neq j
	,\\
	\EE[\norm{a_i}_0]&=s, ~ \forall i \in[k].
	\end{align*}
	Based on the above definitions, three steps achieve the proof of Lemma \ref{lem:jl}:
	\begin{enumerate}
		\item Prove that under certain condition, for any $i, j \in[k]$ with $i \neq j$, with probability at least $1-\frac{2\delta}{k^2}$, we have $\left|\left\langle a_i, a_j\right\rangle\right| \leq  {\varepsilon} $. With probability at least $1-\frac{\delta}{k}$, we have $| \norm{a_i}_2^2-1| \leq \tau$ and $\norm{a_i}_0\leq s+\tau$ for any $i\in[k]$. This is provided in Lemma \ref{lem2}.
		\item By a union bound over all the ${k \choose 2}=k(k-1) / 2$ possible pairs of $(i, j)$ mentioned in Step 1, it concludes that for all $ i, j \in[k]$ with $i \neq j$,  we have $\left|\left\langle a_i, a_j\right\rangle\right| \leq  {\varepsilon} $ with probability at least $1-\delta$. We also have $| \norm{a_i}_2^2-1| \leq \tau$ and $\norm{a_i}_0\leq s+\tau$ for all $i\in[k]$ with probability at least $1-\delta$ by a union bound over all $i\in[k]$.
		\item We normalize $\forall~ a_i\in\cP$ and get $\tilde{\cP}=\{\tilde{a}_1,\tilde{a}_2,\cdots,\tilde{a}_k\}$ where $\norm{\tilde{a}_i}_2=1$ with $i\in[k]$. From $\norm{a_i}_0\leq s+\tau$ and $0\leq \tau<1$ mentioned in Step 2, we can bound $\norm{\tilde{a}_i}_0\leq s$ with $s\in[k]$. Based on Lemma \ref{lem2} and normalized set $\tilde{\cP}$, Theorem \ref{prop3} presents the condition where the feature matrix  $\Phi\in\RR^{k\times d}$ in Lemma \ref{lem:jl} can be constructed by setting $\rows(\Phi) = (\tilde{a}_i)_{i=1}^k$.
	\end{enumerate}
	
	\begin{lemma}\label{lem2}
		Let $0<\delta<1$.	Consider the set $\cP = \{a_1,a_2,\cdots,a_k\}$ described in (\ref{a_il}).
		
		If $0 < \varepsilon\leq \frac{C^2s}{ d}$, by choosing $k\geq\sqrt{\delta}\exp\left({\frac{d\varepsilon^2}{4C^2}}\right)$, we have
		\begin{align}\label{con1}
		\text{  for any  }i, j \in[k],~i \neq j,~~	\left| \langle a_i,a_j\rangle \right| \leq \varepsilon ~\text{ with probability at least } 1-2\delta/k^2.
		\end{align}
		
		If $\varepsilon> \frac{C^2s}{ d}$, by choosing $k\geq\sqrt{\delta}\exp\left({\frac{s\varepsilon}{4}}\right)$, we have
		\begin{align}\label{con2}
		\text{  for any  }i, j \in[k],~i \neq j,~~	\left| \langle a_i,a_j\rangle \right| \leq \varepsilon ~\text{ with probability at least } 1-2\delta/k^2.
		\end{align}
		
		For sufficiently small $\tau$, $0\leq\tau<1$, by choosing $k\geq\frac{\delta}{2} e^{ {\tau^2}/{8}}$, we have
		\begin{align}\label{con4}
		\text{  for any  }i \in[k],~\left| \norm{a_i}_2^2-1\right| \leq \tau ~\text{ with probability at least } 1-\delta/k.
		\end{align} 	
		
		Moreover, by choosing $k\geq{\delta} e^{2 {\tau^2}/{d}}$, we have
		\begin{align}\label{con5}
		\text{  for any  }i \in[k],~\norm{a_i}_0\leq s+\tau~
		\text{ with probability at least } 1-\delta/k.
		\end{align} 	  
	\end{lemma}
	\begin{proof}
		Please refer to Section \ref{proof_oth} for detailed proof.
	\end{proof}
	
	\begin{proposition}\label{prop3}
		Let $0<\delta<1$, $0\leq\tau<1$, $c>1$ and $C'=  {\frac{2c^3}{ (1+\tau)\sqrt{c^2-1}}}$.	Consider the normalized set $\tilde{\cP}=\{\tilde{a}_1,\tilde{a}_2,\cdots,\tilde{a}_k\}$ derived from $\cP$ (\ref{a_il}). 	For sufficiently small $\tau$, we have
		\begin{align} \label{union}
		\text{  for all  }i, j \in[k]~i \neq j,~~	\left| \langle \tilde{a}_i,\tilde{a}_j\rangle \right| \leq \varepsilon,~~ \norm{\tilde{a}_i}_0\leq s 
		~\text{ with probability at least } 1-\delta,
		\end{align}
		by choosing  $k\geq\sqrt{\delta}\exp\left({\frac{d(1+\tau)\varepsilon^2}{4C'}}\right)$ if $0 < \varepsilon\leq \frac{C's}{ d}$.
		If $\varepsilon> \frac{C's}{ d}$, we choose $k\geq\sqrt{\delta}\exp\left({\frac{s(1+\tau)\varepsilon}{4}}\right)$ to achieve (\ref{union}).
		
		Therefore, with probability at least $1-\delta$, the normalized set $\tilde{\cP}$ satisfies that for all $ i, j \in[k],~i \neq j, ~ \langle \tilde{a}_i,\tilde{a}_j\rangle\leq \varepsilon$, $\norm{\tilde{a}_i}_0\leq s$. Hence, the feature matrix  $\Phi\in\RR^{k\times d}$ in Lemma \ref{lem:jl} can be established by choosing $\rows(\Phi) = (\tilde{a}_i)_{i=1}^k$ where $\tilde{a}_i\in\tilde{\cP}$ when $k$ is sufficiently large according to Proposition \ref{prop3}.
	\end{proposition}

	\section{Proof of Lemma \ref{lem2}}\label{proof_oth}
	We first introduce some existential definitions and propositions which are helpful to our proof.
	\begin{definition}\label{def}
		A random variable $X$ with mean $\mu=\mathbb{E}[X]$ is sub-exponential if there are non-negative parameters $(v, \alpha)$ such that
		
		$$
		\mathbb{E}\left[e^{\lambda(X-\mu)}\right] \leq e^{\frac{v^{2} \lambda^{2}}{2}}, \quad \forall ~|\lambda|<\frac{1}{\alpha} .
		$$
	\end{definition}
	
	\begin{proposition}[Sub-exponential tail bound]\label{prop2}
		Assume that $X$ is sub-exponential with parameters $(v, \alpha)$. Then
		
		$$
		\mathbb{P}[|X-\mu |\geq t] \leq \begin{cases}2e^{-\frac{t^{2}}{2 v^{2}}},& ~0 \leq t \leq \frac{v^{2}}{\alpha}, \\ 2e^{-\frac{t}{2 \alpha}},& ~ t>\frac{v^{2}}{\alpha}.\end{cases}
		$$
	\end{proposition}

	For $\forall~ a\in\cP$, each element of $a$ can be taken as the product of two independent random variables, i.e., one is from the Bernoulli distribution and the other is from the Gaussian distribution. Hence, the individual term, i.e., $a_{i\ell}a_{j\ell}$, of $\langle a_i,a_j\rangle = \sum_{\ell=1}^{d}a_{i\ell}a_{j\ell}$ with $\forall a_i, a_j\in\cP,~i\neq j$ can be represented by a random variable $Z_{\ell}$. Specifically, $Z_{\ell} = P_{{\ell}}X_{\ell}Q_{{\ell}}Y_{\ell}$ where $\ell\in[d]$ is the product of independent random variables. Herein, $P_{\ell}$ and $Q_{\ell}$ are independent Bernoulli random variables which take the value $1$ with probability $s/d$ and the value $0$ with probability $1-s/d$. $X_{\ell}$ and $Y_{\ell}$ are independent Gaussian random variables drawn from $\cN(0,1/s)$. For $|\lambda|<m$, we have
	\begin{align}\label{moment}
	\EE[e^{\lambda Z_{\ell}}] &=\sum_{pq\in\{0,1\}}\PP[P_{\ell}Q_{\ell} = pq]\cdot \frac{s}{2\pi}\int_{-\infty}^{\infty}\int_{-\infty}^{\infty}  e^{\lambda (pq)xy}\cdot e^{- {s(x^2+y^2)}/{2}} dxdy\notag\\
	&=\frac{s}{2\pi}\int_{-\infty}^{\infty}\int_{-\infty}^{\infty}  e^{\lambda xy}\cdot e^{-{s(x^2+y^2)}/{2}} dxdy\cdot \left(\frac{s}{d}\right)^2\notag\\
	&\quad+\frac{s}{2\pi}\int_{-\infty}^{\infty}\int_{-\infty}^{\infty}  e^{- {s(x^2+y^2)}/{2}} dxdy\cdot \left(1-\left(\frac{s}{d}\right)^2\right)\notag\\
	&\overset{(i)}{\leq}\frac{s}{2\pi}\cdot\frac{2\pi}{\sqrt{s^2-\lambda^2}}\cdot\left(\frac{s}{d}\right)^2+\frac{s}{2\pi}\cdot \frac{2\pi}{s}\left(1-\left(\frac{s}{d}\right)^2\right)\notag\\
	&\leq\frac{s^3}{d^2\sqrt{s^2-\lambda^2}}+1\notag\\
	&\overset{(ii)}{=} \frac{c^3\lambda^2}{d^2\sqrt{c^2-1}}+1\notag\\
	&\overset{(iii)}{\leq}e^{\frac{c^3\lambda^2}{d^2\sqrt{c^2-1}}}
	\end{align}
	where step (i) comes from
	\begin{align}
	&\int_{- \infty}^{\infty}\int_{- \infty}^{\infty}e^{\lambda(xy)}e^{{-s(x^2+y^2)}/{2}}dxdy \notag\\ =&
	\int_{- \infty}^{\infty}\int_{- \infty}^{\infty}e^{- {s(x-\frac{\lambda}{s} y)^2}/{2}}e^{ {\lambda^2y^2}/{(2s)}}e^{ {-sy^2}/{2}}dxdy\notag \\ =&
	{\sqrt{\frac{2\pi}{s}}}\int_{- \infty}^{\infty}e^{{\lambda^2y^2}/{(2s)}}e^{{-s^2y^2}/{(2s)}}dy \notag\\ =&
	{\sqrt{\frac{2\pi}{s}}}\int_{- \infty}^{\infty}e^{-{y^{2}(s^2-\lambda^2)}/(2s)}dy \notag\\ =&
	\frac{2\pi}{\sqrt{s^2-\lambda^2}},
	\end{align}
	step (ii) is derived by choosing $s=c|\lambda|$, $c>1$, and step (iii) is due to the fact $x+1\leq e^x$.

	Following (\ref{moment}) and Definition \ref{def}, we find that
	\begin{align}
	\EE[e^{\lambda Z_{\ell}}]  \leq e^{\frac{c^3\lambda^2}{d^2\sqrt{c^2-1}}}=e^{\frac{v^2\lambda^2}{2}}, \quad \text { for all }|\lambda|<m ~\text{and}~ v^2 = \frac{2c^3}{d^2\sqrt{c^2-1}}, c>1, 
	\end{align}
	which shows that $Z_{\ell}$ is sub-exponential with parameters $(v_{\ell}, \alpha_{\ell})=(C/d,1/s)$ where $C =  \sqrt{\frac{2c^3}{ \sqrt{c^2-1}}}$ and $c>1$.
	Furthermore, the variable $\sum_{\ell=1}^{d}\left(Z_{\ell}-\EE[Z_{\ell}]\right)$ is sub-exponential with the parameters $\left(v_{*}, \alpha_{*}\right)$, where
	
	$$
	\alpha_{*}:=\max _{\ell=1, \ldots, n} \alpha_{\ell}=\frac{1}{s} \quad \text { and } \quad v_{*}:=\sqrt{\sum_{\ell=1}^{d} v_{\ell}^{2}} .
	$$
	
	Based on the fact $\EE[Z_{\ell}]=0$, the tail bound can be derived from Proposition \ref{prop2},
	\begin{align}
	\mathbb{P}\left[\left| \sum_{\ell=1}^{d} Z_{\ell} \right| \geq t\right] \leq \begin{cases}2e^{-\frac{ t^{2}}{2v_{*}^{2}}}, &  ~0 \leq t \leq \frac{v_{*}^{2}}{ \alpha_{*}}, \\ 2e^{-\frac{ t}{2 \alpha_{*}}},& ~ t>\frac{v_{*}^{2}}{\alpha_{*}}.\end{cases}
	\end{align}
	Thus, we have for two vectors $a_i,a_j\in\cP$ and $i\neq j$,
	\begin{align}\label{aa}
	\mathbb{P}\left[\left| \langle a_i,a_j\rangle \right| \geq t\right] \leq \begin{cases}2e^{-\frac{ dt^{2}}{2C^2}}, & ~0 \leq t \leq \frac{C^2s}{ d}, \\ 2e^{-\frac{ mt}{2  }}, & ~ t>\frac{C^2s}{ d},\end{cases}
	\end{align}
	where $C =  \sqrt{\frac{2c^3}{ \sqrt{c^2-1}}}$ and $c>1$.
	By setting $2e^{-\frac{ dt^{2}}{2C^2}}=2\delta/k^2$, we have $t =\sqrt{\frac{2C^2}{d}\log(\frac{k^2}{\delta})}$. We choose $k\geq\sqrt{\delta}\exp\left({\frac{d\varepsilon^2}{4C^2}}\right)$ such that $t\geq \varepsilon$. Hence, we conclude $\mathbb{P}\left[\left| \langle a_i,a_j\rangle \right| \geq \varepsilon\right] \leq  2\delta/k^2$, which implies the statement (\ref{con1}) when $0 < \varepsilon\leq \frac{C^2s}{ d}$ in Lemma \ref{lem2}. Similar arguments can be applied to the proof of the statement (\ref{con2}) when $\varepsilon> \frac{C^2s}{ d}$ in Lemma \ref{lem2}. The proof of the statement (\ref{con4}) can also be completed by following similar but simpler arguments of proving the statement (\ref{con1}) and (\ref{con2}).

	We are left to the proof of statement (\ref{con5}). For $\forall ~a\in\cP$, the random variable $\norm{a}_0$ obeys the binomial distribution with parameters $d$ and $s/d$, i.e., $\cB(d,s/d)$.  It is the discrete probability distribution of the number of $d$ independent Bernoulli trials which return Boolean-valued outcome: the $\ell$-th ($\ell\in[d]$) element of $a$ is non-zero (with probability $s/d$) or zero (with probability $1-s/d$).
	
	According to the book by Ross \citep{ross2014introduction}, we first introduce several properties of the binomial distribution.
	The cumulative distribution function of binomial distribution $\cB(n,p)$ can be represented by \[\mathbb{F}(k ; n, p)=\mathbb{P}[X \leq k]=\sum_{i=0}^{\lfloor k\rfloor}{n\choose i}p^i(1-p)^{n-i},\]
	where we also have $\mathbb{F}(n-k ; n, 1-p)=1-\mathbb{F}(k ; n, p)$.
	Based on Hoeffding's inequality, $F(k ; n, p)$ can be bounded by
	\[\mathbb{F}(k ; n, p) \leq \exp \left(-2 n\left(p-\frac{k}{n}\right)^2\right).\]
	Hence, the upper tail bound for the random variable $\norm{a}_0$ is given by
	\begin{align}
	\mathbb{P}[\norm{a}_0 \geq m+\tau]=\mathbb{F}(d-s-\tau ;d, 1-\frac{s}{d})\leq\exp\left(\frac{-2\tau^2}{d}\right),
	\end{align}
	where $0\leq \tau<1$. By choosing $k\geq \delta\exp(2\tau^2/d)$, it yields $\mathbb{P}[\norm{a}_0 \geq s+\tau]\leq \frac{\delta}{k}\leq  \exp\left(\frac{-2\tau^2}{d}\right)$. Thus, we completed the proof of statement (\ref{con5}).

	\section{$\poly(s)$-Query Algorithm for $s$-sparsity Case with Noise} \label{noise}
	All results above focus on the noiseless case. We briefly give a discussion on the noisy cases.
	Consider the stochastic misspecified sparse linear bandits where a feature matrix $\Phi\in \R^{k\times d}$, $x_t \in \rows(\Phi)$,  and the reward
	\begin{align}\label{re_n}
	r_{x_t}= \langle {x_t}, \theta^*\rangle + \nu_{x_t}+\eta_t\,
	\end{align}
	where $\nu_{x_t}\in[-\varepsilon,\varepsilon]$ and $\{\eta_t\}$ is a sequence of independent 1-subgaussian random variables.
	
	Based on the reward function (\ref{re_n}) and the notation in Algorithm \ref{alg4}, we start with the approximation error of $f(\theta^*) $:
	\begin{align}\label{error4}
	&|\langle f(a),\hat{\theta}_f\rangle-\langle a,\theta^*\rangle|\notag\\
	\leq&
	|\langle f(a),\hat{\theta}_f\rangle-\langle f(a),f(\theta^*)\rangle|+|\langle f(a),f({\theta^*})\rangle-\langle a,\theta^*\rangle|,\notag\\
	\leq &\left| f(a)^\top G(\rho)^{-1}\sum_{b_t\in\cS}
	\rho(b_t)\nu_{b_t} b_t +  f(a)^\top G(\rho)^{-1} \sum_{b_t\in\cS} \rho(b_t)b_t\eta_t\right|+2\upsilon\notag\\
	\leq &\left| f(a)^\top G(\rho)^{-1}\sum_{b_t\in\cS}
	\rho(b_t)\nu_{b_t} b_t\right| + \left|  f(a)^\top G(\rho)^{-1} \sum_{b_t\in\cS} \rho(b_t)b_t\eta_t\right|+2\upsilon
	\end{align}
	for $\forall ~a\in\rows(\Phi)$.
	
	The first term in (\ref{error4}) can be bounded as 
	\begin{align}\label{e1}
	&\left|f(a)^\top G(\rho)^{-1} \sum_{b_t \in \cS} \rho(b_t) \nu_{b_t}  b_t\right|
	\leq \varepsilon\sum_{b_t \in \cS} \rho(b_t)  \left|f(a)^\top G(\rho)^{-1} b_t\right|\notag \\
	\leq &\varepsilon \sqrt{\left(\sum_{b_t \in \cS} \rho(b_t) \right) f(a)^\top \sum_{b_t \in \cS} \rho(b_t)  G(\rho)^{-1} b_tb_t^\top G(\rho)^{-1} f(a)} \notag\\
	=& \varepsilon \sqrt{\sum_{b_t \in \cS} \rho(b_t)  \norm{f(a)}^2_{G(\rho)^{-1}}} \notag\\
	\leq 	&2\varepsilon \sqrt{p}\,,
	\end{align}
	where is derived from Jensen's inequality and the fact that $\norm{f(a)}^2_{G^{-1}} \leq 2p/t$ for $t$-th time step in Algorithm \ref{alg4}.
	The second term in \ref{error4} can be bounded by standard concentration bounds: with probability at least $1 - 2/(kn)$,
	\begin{align}\label{e2}
	\left|  f(a)^\top G(\rho)^{-1} \sum_{b_t\in\cS} \rho(b_t)b_t\eta_t\right|
	&\leq \norm{f(a)}_{G^{-1}} \sqrt{2 \log\left(kn\right)} \notag\\
	&\leq \sqrt{\frac{4p}{t} \log\left(kn\right)}.
	\end{align}
	
	Combining (\ref{error4}), (\ref{e1}), (\ref{e2}), we have
	\begin{align}
	|\langle f(a),\hat{\theta}_f\rangle-\langle a,\theta^*\rangle|\le  2\varepsilon \sqrt{p} + 
	\sqrt{\frac{4p}{t} \log\left(kn\right)}+2\upsilon.
	\end{align}
	Similarly to the analysis in Section \ref{sec:relax}, we can derive the final approximate error as 
	\begin{align}\label{noise_err}
	&|\langle f(a),\hat{\theta}_f\rangle-\langle a,\theta^*\rangle|\notag\\
	\leq  &C\left((\log(k))^{\frac{1}{4}}\sqrt{\varepsilon} + \sqrt{\frac{p}{t} \log\left(kn\right)}\right).
	\end{align}
	Based on (\ref{noise_err}), the active action set in Algorithm \ref{alg4} in the noise case should be
	\[\cS \leftarrow \left\{a \in \cS : \max_{b \in \cS} \ip{\hat{\theta}_f, b - a} \leq C\left((\log(k))^{\frac{1}{4}}\sqrt{\varepsilon}+\sqrt{\frac{p}{t} \log\left(kn\right)}\right)\right\}.\]

\end{document}